\documentclass{llncs}
\usepackage{makeidx}  

\usepackage{amssymb, amsmath}
\usepackage{graphicx}
\usepackage{stmaryrd}

\usepackage{url}

\newcommand\eqdef{\mathrel{\stackrel{\makebox[0pt]{\mbox{\normalfont\tiny def}}}{=}}}
\begin{document}
\frontmatter          
\pagestyle{headings}  
%
%
\mainmatter              
\title{Multi-task and Lifelong Learning of Kernels}

\author{Anastasia Pentina\inst{1} \and Shai Ben-David\inst{2}}

\institute{Institute of Science and Technology Austria, 3400 Klosterneuburg Austria\\
\email{apentina@ist.ac.at}
\and
University of Waterloo, School of Computer Science, Waterloo ON, Canada\\
\email{shai@uwaterloo.ca}}

\maketitle              

\begin{abstract}
We consider a problem of learning kernels for use in SVM classification in the multi-task and lifelong scenarios and provide generalization bounds on the error of a large margin classifier.
Our results show that, under mild conditions on the family of kernels used for learning, solving several related tasks simultaneously is beneficial over single task learning.
In particular, as the number of observed tasks grows, assuming that in the considered family of kernels there exists one that yields low approximation error on all tasks, the overhead associated with learning such a kernel vanishes and the complexity converges to that of learning when this good kernel is given to the learner.
\end{abstract}
\begin{center}
\textbf{There is a mistake in the conference version of the manuscript: in Theorem 4 on the right hand side there should be $N_{(D,2n)}$ instead of $N_{(D,n)}$. This results in the additional constant $32$ in Theorem 5. \\
}
\end{center}

\section{Introduction}

State-of-the-art machine learning algorithms are able to solve many problems sufficiently well.
However, both theoretical and experimental studies have shown that in order to achieve solutions of reasonable quality they need an access to extensive amounts of training data.
In contrast, humans are known to be able to learn concepts from just a few examples.
A possible explanation may lie in the fact that humans are able to reuse the knowledge they have gained from previously learned tasks for solving a new one, while traditional machine learning algorithms solve tasks in isolation.
This observation motivates an alternative, transfer learning approach.
It is based on idea of transferring information between related learning tasks in order to improve performance.

There are various formal frameworks for transfer learning, modeling different learning scenarios.
In this work we focus on two of them: the multi-task and the lifelong settings.
In the multi-task scenario, the learner faces a fixed set of learning tasks simultaneously and its goal is to perform well on all of them.
In the lifelong learning setting, the learner encounters a stream of tasks and its goal is to perform well on new, yet unobserved tasks.

For any transfer learning scenario to make sense (that is, to benefit from the multiplicity of tasks), there must be some kind of relatedness between the tasks.
A common way to model such task relationships is  through the assumption that there exists some data representation under which learning each of the tasks is relatively easy. 
The corresponding transfer learning methods aim at learning such a representation.

In this work we focus on the case of large-margin learning of kernels. 
We consider sets of tasks and families of kernels and analyze the sample complexity of finding a kernel in a kernel family that allows low expected error on average over the set of tasks (in the multi-task scenario), or in expectation with respect to some unknown task-generating probability distribution (in the lifelong scenario).
We provide generalization bounds for empirical risk minimization learners for both settings.
Under the assumption that the considered kernel family has finite pseudodimension, we show that by learning several tasks simultaneously the learner is guaranteed to have low estimation error with fewer training samples per task (compared to solving them independently).
In particular, if there exists a kernel with low approximation error for all tasks, then, as the number of  observed tasks grows, the problem of learning any specific task with respect to a family of kernels converges to learning when the learner knows a good kernel in advance - the multiplicity of tasks relieves the overhead associated with learning a kernel.
Our assumption on finite pseudodimension of the kernel family is satisfied in many practical cases, like families of Gaussian kernels with a learned covariance matrix,  and linear and convex combinations of a finite set of kernels (see~\cite{Srebro}).
We also show that this is the case for families of all sparse combinations of kernels from a large ``dictionary" of kernels.
\subsection{Related previous work}
\noindent\textbf{Multi-task and Lifelong Learning.} 
A method for learning a common feature representation for linear predictors in the multi-task scenario was proposed in~\cite{Argyriou}.
A similar idea was also used by~\cite{Daume} and extended to the lifelong scenario by~\cite{ELLA}.
A natural extension of representation learning approach was proposed for kernel methods in~\cite{Jebara-2004,Jebara:2011}, where the authors described a method for learning a kernel that is shared between tasks as a combination of some base kernels using maximum entropy discrimination approach.
A similar approach, with additional constraints on sparsity of kernel combinations, was used by~\cite{tnn2011}.
These ideas were later generalized to the case, when related tasks may use slightly different kernel combinations~\cite{GonenKK11,ZhouJH10}, and successfully used in practical applications~\cite{5259,SamCAIP11}.

Despite intuitive attractiveness of the possibility of automatically learning a suitable feature representation compared to learning with a fixed, perhaps high-dimensional or just irrelevant set of features, relatively little is known about its theoretical justifications. 
A seminal systematic theoretical study of the multi-task/lifelong learning settings was done by Baxter in~\cite{Baxter}.
There the author provided sample complexity bounds for both scenarios under the assumption that the tasks share a common optimal hypothesis class.
The possible advantages of these approaches according to Baxter's results depend on the behavior of complexity terms, which, however, due to the generality of the formulation, often can not be inferred easily given a particular setting.
Therefore, studying more specific scenarios by using more intuitive complexity measures may lead to better understanding of the possible benefits of the multi-task/lifelong settings, even if, in some sense, they can be viewed as particular cases of Baxter's result.
Along that line, Maurer in~\cite{Maurer} proved that learning a common low-dimensional representation in the case of lifelong learning of linear least-squares regression tasks is beneficial.

\noindent\textbf{Multiple Kernel Learning.} The problem of multiple kernel learning in the single-task scenario has been theoretically analyzed using different techniques.
By using covering numbers, Srebro et al in~\cite{Srebro} have shown generalization bounds with additive dependence on the pseudodimension of the kernel family.
Another bound with multiplicative dependence on the pseudodimension was presented in~\cite{Ying2009}, where the authors used Rademacher chaos complexity measure.
Both results have a form $O(\sqrt{d/m})$, where $d$ is the pseudodimension of the kernel family and $m$ is the sample size.
%
%
By carefully analyzing the growth rate of the Rademacher complexity in the case of the linear combinations of finitely many kernels with $l_p$ constraint on the weights, Cortes et al in~\cite{DBLP:conf/icml/CortesMR10a} have improved the above results.
In particular, in the case of $l_1$ constraints, the bound from~\cite{Srebro} has a form $O(\sqrt{k/m})$, where $k$ in the total number of kernels, while the bound from~\cite{DBLP:conf/icml/CortesMR10a} is $O(\sqrt{\log(k)/m})$.
%
%
The fast rate analysis of the linear combinations of kernels using local Rademacher complexities was performed by Kloft et al in~\cite{DBLP:journals/jmlr/KloftB12}.

In this work we utilize techniques from~\cite{Srebro}.
It allows us to formulate results that hold for any kernel family with finite pseudodimension and not only for the case of linear combinations, though at the price of potentially suboptimal dependence on the number of kernels in the latter case.
Moreover, additive dependence on the pseudodimension is especially appealing for the analysis of the multi-task and lifelong scenarios, as it allows obtaining bounds where that additional complexity term vanishes as the number of tasks grows and therefore these bounds clearly show possible advantages of transfer learning.
%
%

We start by describing the formal set up and preliminaries in Section~\ref{sec:setup},\ref{sec:cov_numbers} and providing a list of known kernel families with finite pseudodimensions, including our new result for sparse linear combinations, in~\ref{sec:pseudodim}.
In Section~\ref{sec:mt} we provide the proof of the generalization bound for the multi-task case and extend it to the lifelong setting in Section~\ref{sec:ll}.
We conclude by discussion in Section~\ref{sec:conc}.

\section{Preliminaries}
\label{sec:prel}
\subsection{Formal Setup}
\label{sec:setup}

Throughout the paper we denote the input space by $X$ and the output space by $Y=\{-1,1\}$.
We assume that the learner (both in the multi-task and the lifelong learning scenarios) has an access to $n$ tasks represented by the corresponding training sets $\mathbf{z_1},\dots,\mathbf{z_n}\in(X\times Y)^m$, where each $\mathbf{z_i}=\{(x_{i1},y_{i1}),\dots,(x_{im},y_{im})\}$ consists of $m$ i.i.d. samples from some unknown task-specific data distribution $P_i$ over $Z=X\times Y$.
In addition we assume that the learner is given a family $\mathcal{K}$ of kernel functions\footnote{\scriptsize{A function $K:X\times X\rightarrow\mathbb{R}$ is called a kernel, if there exist a Hilbert space $\mathcal{H}$ and a mapping $\phi:X\rightarrow\mathcal{H}$ such that $K(x,x')=\langle\phi(x),\phi(x')\rangle$ for all $x,x'\in X$.}} defined on $X\times X$ and uses the corresponding set of linear predictors for learning.
Formally, for every kernel $K\in\mathcal{K}$ we define $\mathcal{F}_K$ to be such set: 
\begin{equation}
\mathcal{F}_K\eqdef\left\{h:x\mapsto\langle w,\phi(x)\rangle \; |\; \|w\|\leq1, K(x,x')=\langle\phi(x),\phi(x')\rangle\right\}
\end{equation}
and $\mathbb{H}$ to be the union of them: $\mathbb{H}=\cup_{K\in\mathcal{K}}\mathcal{F}_K$.

In the multi-task scenario the data distributions $P_1,\dots,P_n$ are assumed to be fixed and the goal of the learner is to identify a kernel $K\in\mathcal{K}$ that performs well on all of them.
Therefore we would like to bound the difference between the expected error rate over the tasks:
\begin{equation}
er(\mathcal{F}_K) = \frac{1}{n}\sum_{i=1}^n\inf_{h\in\mathcal{F}_K}\mathbf{E}_{(x,y)\sim P_i}\llbracket yh(x)<0\rrbracket
\end{equation}
and the corresponding empirical margin error rate:
\begin{equation}
\widehat{er}^\gamma_z(\mathcal{F}_K) = \frac{1}{n}\sum_{i=1}^n\inf_{h\in\mathcal{F}_K}\frac{1}{m}\sum_{j=1}^m\llbracket y_{ij}h(x_{ij})<\gamma\rrbracket.
\end{equation}
Alternatively the learner may be interested in identifying a particular predictor for every task. 
If we define $\mathcal{F}_K^n=\{\mathbf{h}=(h_1,\dots,h_n): h_i\in\mathcal{F}_K \;\forall i=1\dots n\}$ and $\mathbb{H}^n=\cup_{K}\mathcal{F}_K^n$, then it means finding some $\mathbf{h}\in\mathbb{H}^n$ with low generalization error:
\begin{equation}
er(\mathbf{h}) = \frac{1}{n}\sum_{i=1}^n\mathbf{E}_{(x,y)\sim P_i}\llbracket yh_i(x)<0\rrbracket
\end{equation}
based on its empirical margin performance:
\begin{equation}
\widehat{er}^\gamma_z(\mathbf{h}) = \frac{1}{n}\sum_{i=1}^n\frac{1}{m}\sum_{j=1}^m\llbracket y_{ij}h_i(x_{ij})<\gamma\rrbracket.
\end{equation} 
However, due to the following inequality, it is enough to bound the probability of large estimation error for the second case and a bound for the first one will follow immediately:
\begin{align*}
Pr\left\{z\in Z^{(n,m)}\; \exists \; K\in\mathcal{K}: er(\mathcal{F}_K)>\widehat{er}^\gamma_z(\mathcal{F}_K)+\epsilon\right\}\leq\\
Pr\left\{z\in Z^{(n,m)} \; \exists \; \mathbf{h}\in\mathbb{H}^n: er(\mathbf{h})>\widehat{er}^\gamma_z(\mathbf{h})+\epsilon\right\}.
\end{align*}

For the lifelong learning scenario we adopt the notion of task environment proposed  in~\cite{Baxter} and assume that there exists a set of possible data distributions (i.e. tasks) $\mathcal{P}$ and that the observed tasks are sampled from it i.i.d. according to some unknown distribution $Q$.
The goal of the learner is to find a kernel $K\in\mathcal{K}$ that would work well on future, yet unobserved tasks from the environment $(\mathcal{P},Q)$.
Therefore we would like to bound  the probability of large deviations between the expected error rate on new tasks, given by:  
\begin{equation}
er(\mathcal{F}_K)=\mathbb{E}_{P\sim Q}\inf_{h\in\mathcal{F}_K}\mathbb{E}_{(x,y)\sim P}\llbracket h(x)y<0\rrbracket,
\end{equation}
and the corresponding empirical margin error rate $\widehat{er}^\gamma_z(\mathcal{F}_K)$.

In order to obtain the generalization bounds in both cases we employ the technique of covering numbers.

\subsection{Covering numbers and Pseudodimensions}
\label{sec:cov_numbers}
In this subsection we describe the types of covering numbers we will need and establish their connections to pseudodimensions of kernel families.
\begin{definition}
A subset $\tilde{A}\subset A$ is called an $\epsilon$-cover of $A$ with respect to a distance measure $d$, if for every $a\in A$ there exists a $\tilde{a}\in\tilde{A}$ such that $d(a,\tilde{a})<\epsilon$. The covering number $N_d(A,\epsilon)$ is the size of the smallest $\epsilon$-cover of $A$.
\end{definition}
To derive bounds for the multi-task setting we will use covers of $\mathbb{H}^n$ with respect to $\ell_\infty$ metric associated with a sample $\mathbf{x}\in X^{(n,m)}$: 
\begin{equation}
d^\mathbf{x}_\infty(\mathbf{h},\tilde{\mathbf{h}})=\max_{i=1\dots n}\max_{j=1\dots m}|h_i(x_{ij})-\tilde{h}_i(x_{ij})|<\epsilon.
\end{equation}
The corresponding uniform covering number $N_{(n,m)}(\mathbb{H}^n,\epsilon)$ is given by considering all possible samples $\mathbf{x}\in X^{(n,m)}$:
\begin{equation}
N_{(n,m)}(\mathbb{H}^n,\epsilon)=\max_{\mathbf{x}\in X^{(n,m)}}N_{d^\mathbf{x}_\infty}(\mathbb{H}^n,\epsilon).
\end{equation} 

In contrast, for the lifelong learning scenario we will need covers of the kernel family $\mathcal{K}$ with respect to a probability distribution.
For any probability distribution $P$ over $X\times Y$, we denote its projection on $X$ by $P_X$ and define the following distance between the kernels:
\begin{equation}
D_P(K,\tilde{K})\!=\!\max\{\max_{h\in\mathcal{F}_K}\min_{h'\in\mathcal{F}_{\tilde{K}}}\underset{x\sim P_X}{\mathbb{E}}|h(x)-h'(x)|, \max_{h'\in\mathcal{F}_{\tilde{K}}}\min_{h\in\mathcal{F}_{K}}\underset{x\sim P_X}{\mathbb{E}}|h(x)-h'(x)|\}.
\end{equation}
Similarly, for any set of $n$ distributions $\mathbf{P}=(P_1,\dots,P_n)$ we define:
\begin{equation}
D_\mathbf{P}(K,\tilde{K})=\max_{i=1\dots n}D_{P_i}(K,\tilde{K}).
\end{equation}
The minimal size of the corresponding $\epsilon$-cover of a set of kernels $\mathcal{K}$ we will denote by $N_{D_\mathbf{P}}(\mathcal{K},\epsilon)$ and the corresponding uniform covering number by by $N_{(D,n)}(\mathcal{K},\epsilon)=\max_{(P_1,\dots,P_n)}N_{D_\mathbf{P}}(\mathcal{K},\epsilon)$.

In order to make the guarantees given by the generalization bounds,  that we provide, more intuitively appealing we state them using a natural measure of complexity of kernel families, namely, pseudodimension~\cite{Srebro}:
\begin{definition}
The class $\mathcal{K}$ pseudo-shatters the set of $n$ pairs of points\linebreak $(x_1,x_1'),\dots,(x_n,x_n')$ if there exist thresholds $t_1,\dots,t_n$ such that for any \linebreak
$b_1,\dots,b_n\in\{-1,+1\}$ there exists $K\in\mathcal{K}$ such that $sign(K(x_i,x_i')-t_i)=b_i$. The pseudodimension $d_\phi(\mathcal{K})$ is the largest $n$ such that there exists a set of $n$ pairs pseudo-shattered by $\mathcal{K}$.
\end{definition}
To do so we develop upper bounds on the covering numbers we use in terms of the pseudodimension of the kernel family $\mathcal{K}$.
First, we prove the result for $N_{(n,m)}(\mathbb{H}^n,\epsilon)$ that will be used in the multi-task setting:
\begin{lemma}
For any set $\mathcal{K}$ of kernels bounded by $B$($K(x,x)\leq B$ for all $K\in\mathcal{K}$ and all $x$) with pseudodimension $d_\phi$ the following inequality holds:
\begin{align*}
N_{(n,m)}(\mathbb{H}^n,\epsilon)\leq2^n\Big(\frac{4en^2m^3B}{\epsilon^2d_\phi}\Big)^{d_\phi}\Big(\frac{16mB}{\epsilon^2}\Big)^{\frac{64Bn}{\epsilon^2}\log\big(\frac{e\epsilon m}{8\sqrt{B}}\big)}.
\end{align*}
\label{lemma:N(nm)}
\end{lemma}
In order to prove this result, we first introduce some additional notation.
For a sample $\mathbf{x}=(x_1,\dots,x_m)\in X^m$ we define $l_\infty$ distance between two functions:
\begin{equation}
d_\infty^\mathbf{x}(f_1,f_2)=\max_{i=1\dots m}|f_1(x_i)-f_2(x_i)|.
\end{equation}
Then the corresponding uniform covering number is:
\begin{equation}
N_m(\mathcal{F},\epsilon)=\sup_{\mathbf{x}\in X^m}N_{d_\infty^\mathbf{x}}(\mathcal{F},\epsilon)
\end{equation} 
We also define $l_\infty$ distance between kernels with respect to a sample $\mathbf{x}=(\mathbf{x_1},\dots,\mathbf{x_n})\in X^{(n,m)}$ with the corresponding uniform covering number:
\begin{align*}
D_\infty^\mathbf{x}(K,\hat{K})=\max_{i}|K_{\mathbf{x_i}}-\hat{K}_{\mathbf{x_i}}|_{\infty},\;\;\;
N_{(n,m)}(\mathcal{K},\epsilon)=\sup_{\mathbf{x}\in X^{(n,m)}}N_{D_\infty^\mathbf{x}}(\mathcal{K},\epsilon).
\end{align*}
In contrast, in~\cite{Srebro} the distance between two kernels is defined based on a single sample $\mathbf{x}=(x_1,\dots,x_m)$ of size $m$:
\begin{equation}
D^\mathbf{x}_\infty(K,\hat{K})=|K_\mathbf{x}-\hat{K}_\mathbf{x}|_\infty
\end{equation}
and the corresponding covering number is $N_m(\mathcal{K},\epsilon)$.
Note that this definition is in strong relation with ours: $N_{(n,m)}(\mathcal{K},\epsilon)\leq N_{mn}(\mathcal{K},\epsilon)$, and therefore, by Lemma 3 in~\cite{Srebro}:
\begin{equation}
N_{(n,m)}(\mathcal{K},\epsilon)\leq N_{nm}(\mathcal{K},\epsilon)\leq\Big(\frac{en^2m^2B}{\epsilon d_\phi}\Big)^{d_\phi}
\label{eq:1}
\end{equation} 
for any kernel family $\mathcal{K}$ bounded by $B$ with pseudodimension $d_\phi$.
Now we can prove Lemma~\ref{lemma:N(nm)}:
\begin{proof}[of lemma~\ref{lemma:N(nm)}]
Fix $\mathbf{x}=(\mathbf{x}_1,\dots,\mathbf{x}_n)\in X^{(n,m)}$. 
Define $\epsilon_K=\epsilon^2/4m$ and $\epsilon_F=\epsilon/2$.
Let $\widetilde{\mathcal{K}}$ be an $\epsilon_K$-net of $\mathcal{K}$ with respect to $D_\infty^\mathbf{x}$.
For every $\widetilde{K}\in\widetilde{\mathcal{K}}$ and every $i=1\dots n$ let $\widetilde{\mathcal{F}}^i_{\widetilde{K}}$ be an $\epsilon_F$-net of $\widetilde{\mathcal{F}}_{\widetilde{K}}$with respect to $d_\infty^{\mathbf{x}_i}$.
Now fix some $f\in\mathbb{H}^n$. 
Then there exists a kernel $K$ such that $f=(f_1,\dots,f_n)\in\mathcal{F}_K^n$.
Therefore there exists a kernel $\widetilde{K}\in\widetilde{\mathcal{K}}$ such that $|K_{\mathbf{x}_i}-\widetilde{K}_{\mathbf{x}_i}|_\infty<\epsilon_K$ for every $i$.
By Lemma 1 in~\cite{Srebro} $f_i(x_i)=K_{\mathbf{x}_i}^{1/2}w_i$ for some unit norm vector $w_i$ for every $i$.
Therefore for $\tilde{f}_i(\mathbf{x}_i)\eqdef\widetilde{K}_{\mathbf{x}_i}^{1/2}w_i\in\mathcal{F}_{\widetilde{K}}$ we obtain that:
\begin{align*}
d_\infty^{\mathbf{x}_i}(f_i,\tilde{f}_i)=\max_j |f_i(x_{ij}) - \tilde{f}_i(x_{ij})|\leq ||f_i(\mathbf{x}_i)-\tilde{f}_i(\mathbf{x}_i)||=\\
||K_{\mathbf{x}_i}^{1/2}w_i-\widetilde{K}_{\mathbf{x}_i}^{1/2}w_i||\leq
\sqrt{m|K_{\mathbf{x}_i}-\widetilde{K}_{\mathbf{x}_i}|_\infty}\leq\sqrt{m\epsilon_K}.
\end{align*}
In addition, for every $\tilde{f}_i\in\mathcal{F}_{\widetilde{K}}$ there exists $\tilde{\tilde{f}}_i\in\widetilde{\mathcal{F}}^i_{\widetilde{K}}$ such that $d_\infty^{\mathbf{x}_i}(\tilde{f}_i,\tilde{\tilde{f}}_i)<\epsilon_F$.
Finally, if we define $\tilde{\tilde{f}}=(\tilde{\tilde{f}}_1,\dots,\tilde{\tilde{f}}_n)\in \widetilde{\mathcal{F}}_{\widetilde{K}}^1\times\cdots\times\widetilde{\mathcal{F}}^n_{\widetilde{K}}$, we obtain:
\begin{align*}
d_\infty^\mathbf{x}(f,\tilde{\tilde{f}})=\max_{i}d_\infty^{\mathbf{x}_i}(f_i,\tilde{\tilde{f}}_i)\leq\max_{i}(d_\infty^{\mathbf{x}_i}(f_i,\tilde{f_i}) + d_\infty^{\mathbf{x_i}}(\tilde{f}_i,\tilde{\tilde{f}}_i))<\sqrt{m\epsilon_K}+\epsilon_F=\epsilon.
\end{align*}
The above shows that	$\widetilde{\mathcal{F}_{\mathcal{K}}}=\cup_{\widetilde{K}\in\widetilde{\mathcal{K}}}\widetilde{\mathcal{F}}_{\widetilde{K}}^1\times\cdots\times\widetilde{\mathcal{F}}^n_{\widetilde{K}}$ is an $\epsilon$-net of $\mathbb{H}^n$ with respect to $\mathbf{x}$.
Now the statement follows from~\eqref{eq:1} and the fact that for any $\mathcal{F}_K$ with bounded by $B$ kernel $K$(\cite{Srebro,Anthony}):
\begin{equation}
N_m(\mathcal{F}_K,\epsilon)\leq 2\Big(4mB/\epsilon^2\Big)^{\frac{16B}{\epsilon^2}\log_2\big(\frac{\epsilon em}{4\sqrt{B}}\big)}
\end{equation}
\qed
\end{proof}
Analogously we develop an upper bound on the covering number $N_{(D,n)}(\mathcal{K},\epsilon)$, which we will use for the lifelong learning scenario:
\begin{lemma}
There exists a constant $C$ such that for any kernel family $\mathcal{K}$ bounded by $B$ with pseudodimension $d_\phi$:
\begin{equation}
N_{(D,n)}(\mathcal{K},\epsilon)\leq \Big(Cn^5d^5_\phi\Big(\sqrt{B}/\epsilon\Big)^{17}\Big)^{d_\phi}.
\end{equation}
\label{lemma:LL_cov_num}
\end{lemma}
%
The proof of this result is based on the following lemma that connects sample-based and distribution-based covers of kernel families (for the proof see Appendix~\ref{ap:lemmas}):
\begin{lemma}
For any probability distribution $P$ over $X\times Y$ and any $B$-bounded set of kernels $\mathcal{K}$ with pseudo-dimension $d_\phi$ there exists a sample $\mathbf{x}$ of size $m=cd^2_\phi B^{5/2}/\epsilon^5$ for some constant $c$, such that  for every $K,\tilde{K}$ if $D^{\mathbf{x}}_1(K,\tilde{K})<\epsilon/2$, then $D_P(K,\tilde{K})<\epsilon$ (where $D^\mathbf{x}_1$ is the same as $D_P$, but all expectations over $P$ are substituted by empirical averages over $\mathbf{x}$).
\label{lemma:distr_to_sample}
\end{lemma}
%
%
\begin{proof}[of lemma~\ref{lemma:LL_cov_num}]
Fix some set of probability distributions $\mathbf{P}=(P_1,\dots, P_n)$.
For every $P_i$ denote a sample described by Lemma~\ref{lemma:distr_to_sample} by $\mathbf{x}_i$.
Let $\tilde{\mathcal{K}}$ be an $\epsilon/2n$-cover of $\mathcal{K}$ with respect to $D^\mathbf{x}_1$, where $\mathbf{x}=(\mathbf{x}_1,\dots,\mathbf{x}_n)\in X^{mn}$ and $m=cd^2_\phi B^{5/2}/\epsilon^5$.
Then the following chain of inequalities holds:
\begin{align*}
\max_{h\in\mathcal{F}_K}\min_{h'\in\mathcal{F}_{\tilde{K}}}\frac{1}{mn}\sum_{i=1}^{n}\sum_{j=1}^m|h(x_{ij})-h'(x_{ij})|\leq
\max_h\min_{h'}||h(\mathbf{x})-h'(\mathbf{x})||\leq\\
\max_w||K^{\frac{1}{2}}_\mathbf{x}w-\tilde{K}^{\frac{1}{2}}_\mathbf{x}w||\leq||K_\mathbf{x}^{\frac{1}{2}}-\tilde{K}_\mathbf{x}^{\frac{1}{2}}||_2\leq\sqrt{||K_\mathbf{x}-\tilde{K}_\mathbf{x}||_2}\leq
\sqrt{mn|K_\mathbf{x}-\tilde{K}_\mathbf{x}|_\infty}.
\end{align*}
Consequently, by Lemma 3 in~\cite{Srebro}: 
\begin{equation}
|\tilde{\mathcal{K}}|\leq N(\epsilon/2n, \mathcal{K},D_1^\mathbf{x})\leq \left(\frac{4em^3n^5B}{\epsilon^2 d_\phi}\right)^{d_\phi}=\left(Cn^5d^5_\phi\left(\sqrt{B}/\epsilon\right)^{17}\right)^{d_\phi}
\end{equation}.
It is left to show that $\tilde{\mathcal{K}}$ is an $\epsilon$-cover of $\mathcal{K}$ with respect to $D_\mathbf{P}$.
By definition, for every $K\in\mathcal{K}$ there exists $\tilde{K}\in\tilde{\mathcal{K}}$ such that $D^\mathbf{x}_1(K,\tilde{K})<\epsilon/2n$.
Therefore for every $i=1\dots n$:
\begin{align*}
\max_{h\in\mathcal{F}_K}\min_{h'\in\mathcal{F}_{\tilde{K}}}\frac{1}{m}\sum_{j=1}^m|h(x_{ij})-h'(x_{ij})|\leq \max_{h\in\mathcal{F}_K}\min_{h'\in\mathcal{F}_{\tilde{K}}}\frac{n}{mn}\sum_{i,j}|h(x_{ij})-h'(x_{ij})|\!<\!\frac{\epsilon}{2}.
\end{align*}
Consequently, by Lemma~\ref{lemma:distr_to_sample}, $D_{P_i}(K,\tilde{K})<\epsilon$ for all $i=1\dots n$.
\qed
\end{proof}

\subsection{Pseudodimensions of various families of kernels}
\label{sec:pseudodim}
In \cite{Srebro} the authors have shown the upper bounds on the pseudodimensions of some families of kernels:
\begin{itemize}
\item convex or linear combinations of $k$ kernels have pseudodimension at most $k$
\item Gaussian families with learned covariance matrix in $\mathbb{R}^\ell$ have $d_\phi\leq\ell(\ell+1)/2$
\item Gaussian families with learned low-rank covariance have $d_{\phi}\leq kl\log_2(8ekl)$, where $k$ is the maximum rank of the covariance matrix
\end{itemize}
Here we extend their analysis to the case of sparse combinations of kernels.
\begin{lemma}
Let $K_1,\dots, K_N$ be $N$ kernels and let $\mathcal{K}=\{\sum_{i=1}^N w_iK_i: \; \sum_{i=1}^Nw_i=1 \; \text{and} \; \sum_{i=1}^N[w_i\neq0]\leq k\}$. Then:
\begin{equation}
d_\phi(\mathcal{K})\leq 2k\log(k)+2k\log(4eN)
\end{equation}
\end{lemma}
\begin{proof}
For every kernel $K$ define a function $B_K:X\times X\times \mathbb{R}\rightarrow\{-1,1\}$:
\begin{equation}
B_K(x,\bar{x},t)=sign(K(x,\bar{x})-t)
\end{equation}
and denote a set of such functions for all $K\in\mathcal{K}$ by $\mathcal{B}$.
Then $d_\phi(\mathcal{K})=VCdim(\mathcal{B})$.

For every index set $1\leq i_1<\dots<i_k\leq N$ define $\mathcal{K}_i$ to be a set of all linear combinations of $K_{i_1},\dots,K_{i_k}$.
Then: $\mathcal{K}=\cup_i \mathcal{K}_i$ and $d_\phi(\mathcal{K}_i)\leq k$.
Moreover, there are ${N \choose k}\leq\left(\frac{Ne}{k}\right)^k$ of possible sets of  indices $i$.
Therefore $\mathcal{B}$ can also be seen as a union of at most $\left(\frac{Ne}{k}\right)^k$ sets with VC-dimension at most $k$.
VC-dimension of a union of $r$ classes of VC-dimension at most $d$ is at most $4d\log(2d)+2\log(r)$.
The statement of the lemma is obtained by setting $r=\left(\frac{Ne}{k}\right)^k$ and $d=k$.
\qed
\end{proof}

\section{Multi-task Kernel Learning}
\label{sec:mt}
We start with formulating the result using covering number $N_{(n,m)}(\mathbb{H}^n,\epsilon)$:
\begin{theorem}
For any $\epsilon>0$, if $m>2/\epsilon^2$, we have that:
\begin{equation}
Pr\left\{\exists \mathbf{h}\in\mathbb{H}^n: \;er(\mathbf{h})>\widehat{er}^\gamma_z(\mathbf{h})+\epsilon\right\}\leq 2N_{(n,2m)}(\mathbb{H}^n,\gamma/2)\exp\left(-\frac{nm\epsilon^2}{8}\right).
\end{equation} 
\label{thm:mt_upbound}
\end{theorem}
\begin{proof}
We utilize the standard 3-steps procedure (see Theorem 10.1 in~\cite{Anthony}).
If we denote:
\begin{align*}
Q&=\left\{z\in Z^{(n,m)}: \; \exists \mathbf{h}\in\mathbb{H}^n: \;\; er(\mathbf{h})>\widehat{er}^\gamma_z(\mathbf{h})+\epsilon\right\}\\
R&=\left\{z=(r,s)\in Z^{(n,m)}\times Z^{(n,m)}: \; \exists \mathbf{h}\in\mathbb{H}^n: \;\; \widehat{er}_s(\mathbf{h})>\widehat{er}^\gamma_r(\mathbf{h})+\epsilon/2\right\},
\end{align*}
then according to the symmetrization argument $Pr(Q)\leq2Pr(R)$. 
Therefore, instead of bounding the probability of $Q$, we can bound the probability of $R$.

Next, we define $\Gamma_{2m}$ to be a set of permutations $\sigma$ on the set \linebreak $\{(1,1),\dots,(n,2m)\}$ such that $\{\sigma(i,j), \sigma(i,m+j)\}=\{(i,j), (i,m+j)\}$ for every $1\leq i\leq n$ and $1\leq j\leq m$.
Then $Pr(R)\leq\max_{z\in Z^{(2m,n)}} Pr_\sigma(\sigma z\in R)$.

Now we proceed with the last step - reduction to a finite class.
Fix $z\in Z^{(n,2m)}$ and the corresponding $\mathbf{x}=(x_{ij})\in X^{(n,2m)}$.
Let $T$ be a $\gamma/2$-cover of $\mathbb{H}^n$ with respect to $d^\mathbf{x}_\infty$ and fix $\sigma z\in R$.
By definition there exists $\mathbf{h}\in\mathbb{H}^n$ such that $\widehat{er}_s(\mathbf{h})>\widehat{er}^\gamma_r(\mathbf{h})+ \epsilon/2$, where $(r,s)=\sigma z$.
 We can rewrite it as:
 \begin{align*}
 \frac{1}{n}\sum_{i=1}^n\frac{1}{m}\sum_{j=m+1}^{2m}\llbracket h_i(x_{\sigma(ij)})y_{\sigma(ij)}<0\rrbracket> 
  \frac{1}{n}\sum_{i=1}^n\frac{1}{m}\sum_{j=1}^m\llbracket h_i(x_{\sigma(ij)})y_{\sigma(ij)}<\gamma\rrbracket + \epsilon/2.
 \end{align*}
If we denote by $\tilde{\mathbf{h}}$ the function in the cover $T$ corresponding to $\mathbf{h}$, then the following inequalities hold:
\begin{align*}
\text{if} \;\; \tilde{h}_i(x_{ij})y_{ij}<\frac{\gamma}{2}, \; \text{then} \;\; h_i(x_{ij})y_{ij}<\gamma;\; 
\text{if} \;\; h_i(x_{ij})y_{ij}<0, \; \text{then} \;\; \tilde{h}_i(x_{ij})y_{ij}<\frac{\gamma}{2}.
\end{align*} 
By combining them with the previous inequality we obtain that:
 \begin{align*}
 \frac{1}{n}\sum_{i=1}^n\frac{1}{m}\sum_{j=m+1}^{2m}\llbracket \tilde{h}_i(x_{\sigma(ij)})y_{\sigma(ij)}<\frac{\gamma}{2}\rrbracket\!>\! 
  \frac{1}{n}\sum_{i=1}^n\frac{1}{m}\sum_{j=1}^m\llbracket\tilde{h}_i(x_{\sigma(ij)})y_{\sigma(ij)}<\frac{\gamma}{2}\rrbracket\! + \frac{\epsilon}{2}.
 \end{align*}
Now, if we define the following indicator: $v(\tilde{\mathbf{h}},i,j)=\llbracket\tilde{h}_i(x_{ij})y_{ij}<\gamma/2\rrbracket$, then:
\begin{align*}
\underset{\sigma}{Pr}\!\{\sigma z\in\! R\}\!\leq\!\underset{\sigma}{Pr}\!\left\{\!\exists \tilde{\mathbf{h}}\!\in\! T\!:\frac{1}{n}\!\sum_{i=1}^n\!\frac{1}{m}\!\sum_{j=1}^m\!(v(\tilde{\mathbf{h}},\sigma(i,m\!+\!j))\!-\!v(\tilde{\mathbf{h}},\sigma(i,j)))\!>\!\frac{\epsilon}{2}\!\right\}\\
\leq |T|\max_{\tilde{\mathbf{h}}\in T}\underset{\beta}{Pr}\left\{\frac{1}{n}\sum_{i=1}^n\frac{1}{m}\sum_{j=1}^m|v(\tilde{\mathbf{h}},i,m+j) - v(\tilde{\mathbf{h}},i,j)|\beta_{ij}>\epsilon/2\right\}=(*),
\end{align*}
where $\beta_{ij}$  are independent random variables uniformly distributed over $\{-1,1\}$.
Then $\{|v(\tilde{\mathbf{h}},i,m+j) - v(\tilde{\mathbf{h}},i,j)|\beta_{ij}\}$ are $nm$ independent random variables that take values between $-1$ and $1$ and have zero mean.
Therefore by Hoeffding's inequality:
\begin{align*}
(*)\leq |T|\exp\left(-\frac{2(nm)^2\epsilon^2/4}{mn\cdot4}\right)=|T|\exp\left(-\frac{nm\epsilon^2}{8}\right).
\end{align*}
By noting that $|T|\leq N_{(n,2m)}(\mathbb{H}^n,\gamma/2)$, we conclude the proof of Theorem~\ref{thm:mt_upbound}.
\qed
\end{proof}
By using the same technique as for proving Theorem~\ref{thm:mt_upbound}, we can obtain a lower bound on the difference between the empirical error rate $\widehat{er}^\gamma_z(\mathbf{h})$ and the expected error rate with double margin:
\begin{equation}
er^{2\gamma}(\mathbf{h})=\frac{1}{n}\sum_{i=1}^n\mathbf{E}_{(x,y)\sim P_i}\llbracket yh_i(x)<2\gamma\rrbracket.
\end{equation}
\begin{theorem}
For any $\epsilon>0$, if $m>2/\epsilon^2$, the following holds:
\begin{equation}
Pr\!\left\{\exists \mathbf{h}\in\mathbb{H}^n: er^{2\gamma}(\mathbf{h})<\widehat{er}^\gamma_z(\mathbf{h})-\epsilon\right\}\!\leq\! 2N_{(n,2m)}(\mathbb{H}^n,\gamma/2)\exp\left(\!-\frac{nm\epsilon^2}{8}\right)\!.
\end{equation} 
\label{thm:mt_lowbound}
\end{theorem}
Now, by combining Theorems~\ref{thm:mt_upbound},~\ref{thm:mt_lowbound} and Lemma~\ref{lemma:N(nm)} we can state the final result for the multi-task scenario in terms of pseudodimensions:
\begin{theorem}
For any probability distributions $P_1,\dots,P_n$ over $X\times\{-1,+1\}$, any kernel family $\mathcal{K}$, bounded by $B$ with pseudodimension $d_{\phi}$, and any fixed $\gamma>0$, for any $\epsilon>0$, if $m>2/\epsilon^2$, then, for a sample $z$ generated by $\Pi_{i=1}^n (P_i)^m$:
\begin{equation}
Pr\left\{\forall\; \mathbf{h}\in\mathbb{H}^n \;\; er^{2\gamma}(\mathbf{h})+\epsilon\geq \widehat{er}^\gamma_z(\mathbf{h})\geq er(\mathbf{h})-\epsilon\right\}\geq 1-\delta,
\end{equation}
where

\begin{equation}
\epsilon= \sqrt{8\frac{\frac{2\log2-\log\delta}{n}+\log2+\frac{d_{\phi}}{n}\log\frac{128en^2m^3B}{\gamma^2d_\phi}+\frac{256B}{\gamma^2}\log\frac{\gamma em}{8\sqrt{B}}\log\frac{128mB}{\gamma^2}}{m}}.
\end{equation}
\label{thm:mtmkl}
\end{theorem}

\noindent \textbf{Discussion:}  The most significant implications of this result are for the case where there exists some kernel $K \in {\cal K}$ that has low approximation error for each of the tasks $P_i$ (this is what makes the tasks "related" and, therefore, the multi-task approach advantageous). In such a case, the kernel that minimizes the average error over the set of tasks is a useful kernel for each of these tasks. 
\begin{enumerate}
\item Maybe the first point to note about the above generalization result is that as the number of tasks ($n$) grows, while the number of examples per task ($m$) remains constant, the error bound behaves like the bound needed to learn with respect to a single kernel. That is, if a learner wishes to learn some specific task $P_i$, and all the learner knows is that in the big family of kernels ${\cal K}$, \emph{there exists } some useful kernel $K$ for $P_i$ that is also good on average over the other tasks, then the training samples from the other tasks allow the learner of $P_i$ to learn as if he had access to a specific good kernel $K$.
\item Another worthwhile consequence of the above theorem is that it shows the usefulness of an empirical risk minimization approach. Namely,
\begin{corollary} Let $\widehat{\mathbf{h}}$ be a minimizer, over $\mathbb{H}^n$, of the empirical $\gamma$-margin loss, $\widehat{er}^\gamma_z(\mathbf{h})$. Then for any $\mathbf{h}^*\in\mathbb{H}^n$ (and in particular for a minimizer over $\mathbb{H}^n$ of the true $2\gamma$-loss $er^{2\gamma}(\mathbf{h})$):
\[ er({\widehat h}) \leq er^{2 \gamma}(h^*) +2\epsilon . \]
\label{cor:mt}
\end{corollary}
\vspace*{-2.5em}
\begin{proof}
The result is implied by the following chain of inequalities:
\[ er ({\widehat h}) - \epsilon \leq_1 {\widehat er^{\gamma}}({\widehat h}) \leq_2 {\widehat er^{\gamma}}( h^*) \leq_3 er^{2\gamma} (h^*) +\epsilon\]

where $(\leq_1)$ and $(\leq_3)$ follow from the above theorem and $(\leq_2)$ follows from the definition of an empirical risk minimizer.
\qed
\end{proof}
\end{enumerate}

\section{Lifelong Kernel Learning}
\label{sec:ll}

In this section we generalize the results of the previous section to the case of lifelong learning in two steps.
First, note that by using the same arguments as for proving Theorem~\ref{thm:mt_upbound} we can obtain a bound on the difference between $\widehat{er}^{2\gamma}_z(\mathcal{F}_K)$ and:
\begin{equation}
\widehat{er}^{\gamma}_{\mathbf{P}}(\mathcal{F}_K)=\frac{1}{n}\sum_{i=1}^n\inf_{h\in\mathcal{F}_K}\mathbb{E}_{(x,y)\sim P_i}\llbracket h(x)y<\gamma\rrbracket.
\end{equation}
Therefore the only thing that is left is a bound on the difference between $er(\mathcal{F}_K)$ and $\widehat{er}^{\gamma}_{\mathbf{P}}(\mathcal{F}_K)$. 

We will use the following notation:
\begin{align*}
er_P(\mathcal{F}_K)=\inf_{h\in\mathcal{F}_K}\mathbb{E}_{(x,y)\sim P}\llbracket h(x)y<0\rrbracket,\;\;\;
er^\gamma_P(\mathcal{F}_K)=\inf_{h\in\mathcal{F}_K}\mathbb{E}_{(x,y)\sim P}\llbracket h(x)y<\gamma\rrbracket
\end{align*} 
and proceed in a way analogous to the proof of Theorem~\ref{thm:mt_upbound}.
First, if we define:
\begin{align*}
Q=\{\mathbf{P}=(P_1,\dots,P_n)\in\mathcal{P}^n \; \exists\mathcal{F}_K: \;\; er(\mathcal{F}_K)>\widehat{er}^\gamma_{\mathbf{P}}(\mathcal{F}_K)+\epsilon\}\\
R=\{z=(r,s)\in\mathcal{P}^{2n} \; \exists\mathcal{F}_K: \;\; \widehat{er}_s(\mathcal{F}_K)>\widehat{er}^\gamma_r(\mathcal{F}_K)+\epsilon/2\},
\end{align*}
then according to the symmetrization argument $Pr(Q)\leq2Pr(R)$.

Now, if we define $\Gamma_{2n}$ to be a set of permutations $\sigma$ on a set $\{1,2,\dots,2n\}$, such that $\{\sigma(i),\sigma(n+i)\}=\{i,n+i\}$ for all $i=1\dots n$, we obtain that $Pr(R)\leq\max_{z}Pr_\sigma(\sigma z\in R)$, if $n>2/\epsilon^2$.
So, the only thing that is left is reduction to a finite class.

Fix $z$ and denote by $\tilde{\mathcal{K}}\subset\mathcal{K}$ a set of kernels, such that for every $K\in\mathcal{K}$ there exists a $\tilde{K}\in\tilde{\mathcal{K}}$ such that:
\begin{equation}
er^\gamma_{P_i}(\mathcal{F}_K)+\epsilon/8\geq er^{\gamma/2}_{P_i}(\mathcal{F}_{\tilde{K}})\geq er_{P_i}(\mathcal{F}_K)-\epsilon/8 \;\; \forall i=1\dots 2n.
\label{cond}
\end{equation}
Then, if $\mathcal{F}_K$ is such that $\widehat{er}_s(\mathcal{F}_K)>\widehat{er}^\gamma_r(\mathcal{F}_K)+\epsilon/2$, then the corresponding $\tilde{K}$ satisfies $\widehat{er}^{\gamma/2}_s(\mathcal{F}_{\tilde{K}})>\widehat{er}^{\gamma/2}_r(\mathcal{F}_{\tilde{K}})+\epsilon/4$.
Therefore:
\begin{align*}
Pr_\sigma\{\sigma z\in R\}\leq Pr_\sigma\left\{\exists K\in\tilde{\mathcal{K}}: \frac{1}{n}\sum_{i=1}^n(er^{\gamma/2}_{P_{\sigma(n+i)}}(\mathcal{F}_K)-er^{\gamma/2}_{P_{\sigma(i)}}(\mathcal{F}_K))>\epsilon/4\right\}\leq\\
|\tilde{\mathcal{K}}|\max_{K\in\tilde{\mathcal{K}}}Pr_\sigma\left\{\frac{1}{n}\sum_{i=1}^n(er^{\gamma/2}_{P_{\sigma(n+i)}}(\mathcal{F}_K)-er^{\gamma/2}_{P_{\sigma(i)}}(\mathcal{F}_K))>\epsilon/4\right\}=\\
|\tilde{\mathcal{K}}|\max_{K\in\tilde{\mathcal{K}}}Pr_\beta\left\{\frac{1}{n}\sum_{i=1}^n|er^{\gamma/2}_{P_{n+i}}(\mathcal{F}_K)-er^{\gamma/2}_{P_{i}}(\mathcal{F}_K)|\beta_{i}>\epsilon/4\right\}=(*),
\end{align*}
where $\beta_i$ are independent random variables uniformly distributed over $\{-1,+1\}$.
As in the previous section, $\{|er^{\gamma/2}_{P_{n+i}}(\mathcal{F}_K)-er^{\gamma/2}_{P_{i}}(\mathcal{F}_K)|\beta_{i}\}$ are $n$ independent random variables that take values between $-1$ and $1$ and have zero mean.
Therefore by applying Hoeffding's inequality we obtain:
\begin{equation}
(*)\leq |\tilde{\mathcal{K}}|\exp\left(-\frac{2n^2\epsilon^2/16}{4n}\right)=|\tilde{\mathcal{K}}|\exp\left(-\frac{n\epsilon^2}{32}\right).
\label{LL:end}
\end{equation}
To conclude the proof we need to understand how $|\tilde{\mathcal{K}}|$ behaves.
For that we prove the following lemma:
\begin{lemma}
For any set of probability distributions $\mathbf{P}=(P_1,\dots,P_{2n})$ there exists $\tilde{\mathcal{K}}$ that satisfies condition of equation~\eqref{cond} and $|\tilde{\mathcal{K}}|\leq N_{(D,2n)}(\mathcal{K},\epsilon\gamma/16) $.
\label{lemma:tildeK}
\end{lemma}
\begin{proof}
Fix a set of distributions $\mathbf{P}=(P_1,\dots,P_{2n})$ and denote by $\tilde{\mathcal{K}}$ an $\epsilon\gamma/16$-cover of $\mathcal{K}$ with respect to $D_\mathbf{P}$.
Then $|\tilde{\mathcal{K}}|\leq N_{(D,2n)}(\mathcal{K},\epsilon\gamma/16)$.
By definition of a cover for any kernel $K\in\mathcal{K}$ there exists $\tilde{K}\in\tilde{\mathcal{K}}$ such that $D_\mathbf{P}(K,\tilde{K})<\epsilon\gamma/16$.
Equivalently, it means that for every $K\in\mathcal{K}$ there exists $\tilde{K}\in\tilde{\mathcal{K}}$ such that the following two conditions hold for every $i=1\dots 2n$:
\begin{align}
1. \forall \; h\in\mathcal{F}_{K} \;\; \exists h'\in\mathcal{F}_{\tilde{K}}: \;\; \mathbb{E}_{(x,y)\sim P_i}(|h(x)-h'(x)|)<\frac{\epsilon\gamma}{16},\\
2. \forall \; h'\in\mathcal{F}_{\tilde{K}} \;\; \exists h\in\mathcal{F}_{K}: \;\;\mathbb{E}_{(x,y)\sim P_i}(|h(x)-h'(x)|)<\frac{\epsilon\gamma}{16}.
\end{align}
Fix some $K$ and the corresponding kernel $\tilde{K}$ from the cover and take any $P_i$.
By Markov's inequality applied to the first condition we obtain that for every $h\in\mathcal{F}_K$ there exists a $h' \in\mathcal{F}_{\tilde{K}}$ such that $Pr\{x\sim P_i:\; |h(x)-h'(x)|>\gamma/2\}<\epsilon/8$.
Then $er^{\gamma/2}_{P_i}(h')\leq er^{\gamma}_{P_i}(h)+\epsilon/8$.
By applying the same argument to the second condition we conclude that
for every $h'\in\mathcal{F}_{\tilde{K}}$ there exists a $h \in\mathcal{F}_{K}$ such that $Pr\{x\sim P_i:\; |h(x)-h'(x)|>\gamma/2\}<\epsilon/8$.
Then $er_{P_i}(h)\leq er^{\gamma/2}_{P_i}(h')+\epsilon/8$.
By definition of infinum $\widehat{er}^{2\gamma}_z(\mathcal{F}_K)$ for every $\delta$ there exists $h\in\mathcal{F}_K$ such that $er^\gamma_{P_i}(\mathcal{F}_K)+\delta>er^\gamma_{P_i}(h)\geq er^\gamma_{P_i}(\mathcal{F}_K)$.
By above construction for such $h$ there exists $h'\in\mathcal{F}_{\tilde{K}}$ such that $er^{\gamma}_{P_i}(h)\geq er^{\gamma/2}_{P_i}(h')-\epsilon/8\geq er^{\gamma/2}_{P_i}(\mathcal{F}_{\tilde{K}})-\epsilon/8$.
By combining these inequalities we obtain that for every $\delta>0$ $er^\gamma_{P_i}(\mathcal{F}_K)+\delta>er^{\gamma/2}_{P_i}(\mathcal{F}_{\tilde{K}})-\epsilon/8$, or, equivalently, $er^\gamma_{P_i}(\mathcal{F}_K)\geq er^{\gamma/2}_{P_i}(\mathcal{F}_{\tilde{K}})-\epsilon/8$.
Analogously we can get that $er^{\gamma/2}_{P_i}(\mathcal{F}_{\tilde{K}})\geq er_{P_i}(\mathcal{F}_K)-\epsilon/8$.
So, we obtain condition~\eqref{cond}.
\qed
\end{proof}
By combining the above Lemma with~\eqref{LL:end} we obtain the following result (the second inequality can be obtain in a similar manner):
\begin{theorem}
For any $\epsilon>0$, if $n>2/\epsilon^2$, the following holds:
\begin{align*}
Pr\left\{\exists K\in\mathcal{K}:\;\; er(\mathcal{F}_K)>\widehat{er}^\gamma_{\mathbf{P}}(\mathcal{F}_K)+\epsilon\right\}\leq 2N_{(D,2n)}(\mathcal{K},\epsilon\gamma/16)\exp\left(-\frac{n\epsilon^2}{32}\right),\\
Pr\left\{\exists K\in\mathcal{K}: er^{2\gamma}(\mathcal{F}_K)<\widehat{er}^\gamma_{\mathbf{P}}(\mathcal{F}_K)-\epsilon\right\}\leq 2N_{(D,2n)}(\mathcal{K},\epsilon\gamma/16)\exp\left(-\frac{n\epsilon^2}{32}\right).
\end{align*}
\label{thm:ll_upbound}
\end{theorem}
Note that by exactly following the proof of Theorem~\ref{thm:mt_upbound} one can obtain that:
\begin{equation*}
Pr\left\{\exists K\in\mathcal{K}\;\; \widehat{er}^{\gamma/2}_\mathbf{P}(\mathcal{F}_K)-\widehat{er}^\gamma_z(\mathcal{F}_K)>\frac{\epsilon}{2}\right\}\!<2N_{(n,2m)}(\mathbb{H}^n,\gamma/4)\exp\left(-\frac{nm\epsilon^2}{32}\right).
\end{equation*}
Therefore, by combining the above result with its equivalent in the opposite direction with Theorem~\ref{thm:ll_upbound} and Lemmas~\ref{lemma:N(nm)} and~\ref{lemma:LL_cov_num} we obtain the final result for the lifelong kernel learning:
\begin{theorem}
For any task environment, any kernel family $\mathcal{K}$, bounded by $B$ with pseudodimension $d_\phi$, any fixed $\gamma>0$ and any $\epsilon>0$, if $n>8/\epsilon^2$ and $m>8/\epsilon^2$, then:
\begin{equation}
Pr\left\{\forall K\in\mathcal{K}\; er^{2\gamma}(\mathcal{F}_K)+\epsilon\geq\widehat{er}^\gamma_z(\mathcal{F}_K)\geq er(\mathcal{F}_K)-\epsilon\right\}\geq1-\delta,
\end{equation}
where
\begin{align*}
\delta=2^{n+2}\left(\frac{512en^2m^3B}{\gamma^2d_{\phi}}\right)^{d_\phi}\left(\frac{512mB}{\gamma^2}\right)^{\frac{1024Bn}{\gamma^2}\log\left(\frac{e\gamma m}{16\sqrt{B}}\right)}\exp\left(-\frac{nm\epsilon^2}{32}\right)+\\
4\left(32Cn^5d_\phi^5\left(\frac{64\sqrt{B}}{\epsilon\gamma}\right)^{17}\right)^{d_\phi}\exp\left(-\frac{n\epsilon^2}{128}\right).
\end{align*}
\label{thm:llmkl}
\end{theorem}

\noindent \textbf{Discussion:}  As for the multi-task case, the most significant implications of this result are for the case where there exists some kernel $K \in {\cal K}$ that has low approximation error for all tasks in the environment. In such a case, the kernel that minimizes the average error over the set of observed tasks is a useful kernel for all the tasks. 
\begin{enumerate}
\item First, note, that the only difference between Theorem~\ref{thm:llmkl} and Theorem~\ref{thm:mtmkl} is the presence of the second term. This additional complexity comes from the fact that for the lifelong learner we are bounding the expected error on new, yet unobserved tasks. Therefore we have to pay additionally for not knowing exactly what these new tasks are going to be. 
\item Second, the behavior of the above result is similar to that of Theorem~\ref{thm:mtmkl} in the limit of infinitely many observed tasks ($n\rightarrow\infty$). In this case, the second term vanishes, because by observing large enough amount of tasks the learner gets the full knowledge about the task environment. The first term behaves exactly the same as the one in Theorem~\ref{thm:mtmkl}: its part that depends on $d_\phi$ vanishes and therefore it converges to the complexity of learning one task as if the learner would know a good kernel in advance.
\item This theorem also shows the usefulness of an empirical risk minimization approach as we can obtain a corollary of exactly the same form as Corollary~\ref{cor:mt}. 
\end{enumerate}
\section{Conclusions}
\label{sec:conc}

Multi-task and lifelong learning have been a topic of significant interest of research in recent years and attempts for solving these problems in different directions have been made.
Methods of learning kernels in these scenarios have been shown to lead to effective algorithms and became popular in applications.
In this work, we have established sample complexity error bounds that justify this approach.
Our results show that, under mild conditions on the used family of kernels,  by solving multiple tasks jointly the learner can "spread out" the overhead associated with learning a kernel and as the number of observed tasks grows, the complexity converges to that of learning when a good kernel was known in advance.
This work constitutes a step forward better understanding of the conditions under which multi-task/lifelong learning is beneficial.
\subsubsection*{Acknowledgments.} This work was in parts funded by the European Research Council under the European Union's 
Seventh Framework Programme (FP7/2007-2013)/ERC grant agreement no 308036.
%
%

%
\appendix
\section{Proof of lemma~\ref{lemma:distr_to_sample}}
\label{ap:lemmas}
%
Define $G=\left\{g:X\rightarrow[0,1]: g(x)=\frac{|h(x)-h'(x)|}{\sqrt{B}} \;\text{for some}\; h,h'\in\cup\mathcal{F}_K\right\}$. 
Then \linebreak (using Lemma 2 and 3 in~\cite{Bartlett} and Theorem 1 in~\cite{Srebro}):
\begin{align*}
Pr\left\{\mathbf{x}\in X^m: \; \exists K,\tilde{K}: |D_1^\mathbf{x}(K,\tilde{K})-D_P(K,\tilde{K})|>\epsilon/2\right\}\leq\\
 Pr\!\left\{\!\mathbf{x}\!\in\! X^m\!:\exists h,h'\!\in\!\cup\mathcal{F}_K\!:\! \left|\frac{1}{m}\!\sum_{i=1}^m\!|h(x_i)\!-\!h'(x_i)|-\!\!\underset{(x,y)\sim P}{\mathbb{E}}\!|h(x)\!-\!h'(x)|\right|\!>\!\frac{\epsilon}{2}\right\}=\\
Pr\left\{\mathbf{x}\in X^m: \;\; \exists g,g'\in G: \left|\frac{1}{m}\sum_{i=1}^m g(x_i)-\mathbb{E}_{(x,y)\sim P}g(x)\right|>\epsilon/2\sqrt{B}\right\}\leq\\
4\!\max_\mathbf{x}\! N(\frac{\epsilon/32}{\sqrt{B}},G,d^\mathbf{x}_1)e^{-\frac{\epsilon^2m}{512B}}\!\leq\!
 4\max_\mathbf{x}N(\frac{\epsilon}{64\sqrt{B}},\cup\mathcal{F}_K/\sqrt{B}, d^\mathbf{x}_1)^2e^{-\epsilon^2m/512B}=\\
4\max_\mathbf{x}N(\epsilon/64,\cup\mathcal{F}_K, d^\mathbf{x}_1)^2e^{-\epsilon^2m/512B}\leq 4\max_\mathbf{x}N(\epsilon/64,\cup\mathcal{F}_K, d^\mathbf{x}_\infty)^2e^{-\epsilon^2m/512B}\leq\\
4\cdot 4\cdot N(\mathcal{K},\epsilon^2/(64^2\cdot4m))^2\cdot\left(\frac{16mB64^2}{\epsilon^2}\right)^{\frac{2\cdot64^3B}{\epsilon^2}\log\left(\frac{\epsilon em}{64*8\sqrt{B}}\right)}e^{-\epsilon^2m/512B}\leq\\
 16\left(\frac{2^{14}em^3B}{\epsilon^2d_{\phi}}\right)^{2d_\phi}\left(\frac{2^{16}mB}{\epsilon^2}\right)^{\frac{2^{19}B}{\epsilon^2}\log\left(\frac{\epsilon em}{2^9\sqrt{B}}\right)}e^{-\epsilon^2m/512B}=(**)
\end{align*}
For big enough $m$ $(**)$ is less than 1, which means that there is a sample $\mathbf{x}\in X^m$ such that for all kernels $K,\tilde{K}$ we have $|D_1^\mathbf{x}(K,\tilde{K})-D_P(K,\tilde{K})|\leq\epsilon/2$.
More precisely, $m$ should be bigger than $cd^2_\phi B^{5/2}/\epsilon^5$ for some constant $c$.
\qed
\end{document}